\documentclass[journal]{IEEEtran}

\ifCLASSINFOpdf
\else
   \usepackage[dvips]{graphicx}
\fi
\usepackage{url}

\usepackage{graphicx}
\usepackage{amsfonts}
\usepackage{amsmath, amsthm, amssymb}
\usepackage{xcolor}
\usepackage{booktabs}
\usepackage{multirow}
\usepackage{balance}
\newtheorem{lemma}{Lemma}

\begin{document}

\title{Orthogonal Subspace Projection for Continual Machine Unlearning via SVD-Based LoRA}

\author{Yogachandran Rahulamathavan, Nasir Iqbal, Juncheng Hu, and Sangarapillai Lambotharan, \IEEEmembership{Senior Member, IEEE}
\thanks{Y. Rahulamathavan, N. Iqbal and S. Lambotharan are with Institute for Digital Technologies, Loughborough University, UK. E-mails: \{y.rahulamathavan, n.iqbal, s.lambotharan\}@lboro.ac.uk}
\thanks{J. Hu is with the School of Artificial Intelligent Medicine, Guilin Medical University, Guilin 541199, CN. E-mail: 112018015@glmu.edu.cn}
\thanks{\ The authors wish to acknowledge the financial support of the Engineering and Physical Sciences Research Council (EPSRC) under the grants Platform for Driving Ultimate Connectivity (TITAN) (EP/X04047X/1 and EP/Y037243/1) and PerCom (EP/X012301/1).}
}

\markboth{Journal of \LaTeX\ Class Files, Vol. 14, No. 8, August 2015}
{Shell \MakeLowercase{\textit{et al.}}: Bare Demo of IEEEtran.cls for IEEE Journals}
\maketitle

\begin{center}
\footnotesize
This work has been submitted to the IEEE for possible publication. Copyright may be transferred without notice, after which this version may no longer be accessible.
\end{center}

\begin{abstract}
Continual machine unlearning aims to remove the influence of data that should no longer be retained, while preserving the usefulness of the model on everything else. This setting becomes especially difficult when deletion requests arrive sequentially, because the model must repeatedly adapt without erasing previously retained knowledge. Low-Rank Adaptation (LoRA) offers an efficient way to implement such updates, but naively combining many sequential LoRA modules leads to parameter collision, causing \textit{strong interference} between tasks.
We propose a static alternative based on Singular Value Decomposition (SVD)-guided orthogonal subspace projection. Our method constrains each new LoRA update during training so that it lies in the orthogonal complement of the subspaces used by earlier unlearning tasks. This preserves task isolation without requiring dynamic routing at deployment. Experiments on CIFAR-100 with ResNet-20 and on MNIST show stable behavior across long sequences of unlearning tasks. After thirty sequential unlearning tasks, state-of-the-art static fusion reduces retained accuracy from 60.39\% to 12.70\%, whereas the proposed in-training constrained optimization maintains baseline performance ($\sim$58.1\%) while preserving strong unlearning efficacy.
\end{abstract}
\begin{IEEEkeywords}
Machine Unlearning, Low Rank Adapters (LoRA), Privacy.
\end{IEEEkeywords}

\section{Introduction}

\IEEEPARstart{T}{he} ``right to be forgotten," formalized by regulations such as the GDPR, has made machine unlearning an increasingly important capability for deployed learning systems \cite{bourtoule2021machine}. In many settings, deletion requests do not arrive once, but repeatedly over time. A useful continual unlearning method must therefore do more than forget the current target data: it must also preserve the model's behavior on retained data and remain practical to deploy after many rounds of updates.

LoRA\cite{hu2021lora} is attractive in this context because it allows each unlearning step to be represented by a small set of trainable parameters rather than a full retraining of the model. However, this efficiency comes with a challenge. When many LoRA updates are learned sequentially and then fused into a single model, their effects can overlap in the same parameter directions. As the number of tasks grows, this overlap leads to parameter collision, which in turn degrades retained performance and can produce catastrophic forgetting. This challenge is closely related to the broader interference problem studied in continual learning \cite{kirkpatrick2017overcoming,lopez2017gradient,zeng2019continual}, while classical machine unlearning methods based on selective retraining remain costly under repeated deletion requests \cite{golatkar2020eternal,koh2017understanding}.

One response to this problem is to avoid static fusion entirely. Recent methods such as AC-LoRA \cite{aclora2025} and I-LoRA \cite{ilora2024} instead rely on dynamic Mixture-of-Experts (MoE) routing \cite{dou2023loramoe}, where a routing network decides which adapter should be emphasized for each input. This helps isolate task-specific behavior and reduces interference. Yet the price of this isolation is paid at deployment time: all adapters must remain available, additional routing parameters must be stored, and each inference requires a routing decision before combining the adapter outputs. These costs are manageable in some settings, but they weaken the appeal of continual unlearning when static and lightweight deployment is desired.

In this paper, we pursue a different goal: can sequential unlearning updates remain separable \emph{without} introducing a dynamic routing mechanism at inference time? Our answer is to enforce separation during training itself. We formulate continual unlearning as an \textit{in-training constrained optimization problem} in which each new LoRA update is restricted to the orthogonal complement of the subspaces already occupied by earlier tasks. To construct this constraint, we use SVD to summarize the dominant directions of previous updates and accumulate them into a projection matrix.

This perspective yields two practical advantages. First, it directly targets the source of interference by controlling the geometry of the learned updates rather than resolving collisions after they occur. Second, because the resulting updates occupy orthogonal subspaces, they can be fused back into the base model for static deployment, removing the need for a router and dynamic adapter selection. We also show that embedding the projection into the forward pass behaves as a form of projected gradient descent, which avoids the sub-optimality often introduced by post-hoc projection after unconstrained training.

Experiments on MNIST and CIFAR-100 using ResNet-20 demonstrate that the proposed approach remains stable over long sequences of unlearning tasks. In particular, under $N=30$ sequential requests, conventional static fusion collapses retained accuracy to 12.7\% (from the baseline accuracy of 60.39\%), whereas our in-training constrained optimization preserves the baseline level ($\sim$58.1\%) while maintaining strong unlearning efficacy.

\begin{figure*}[t]
  \centering
    \boxed{\includegraphics[width=0.8\linewidth, trim=0cm 2cm 0cm 2cm, clip]{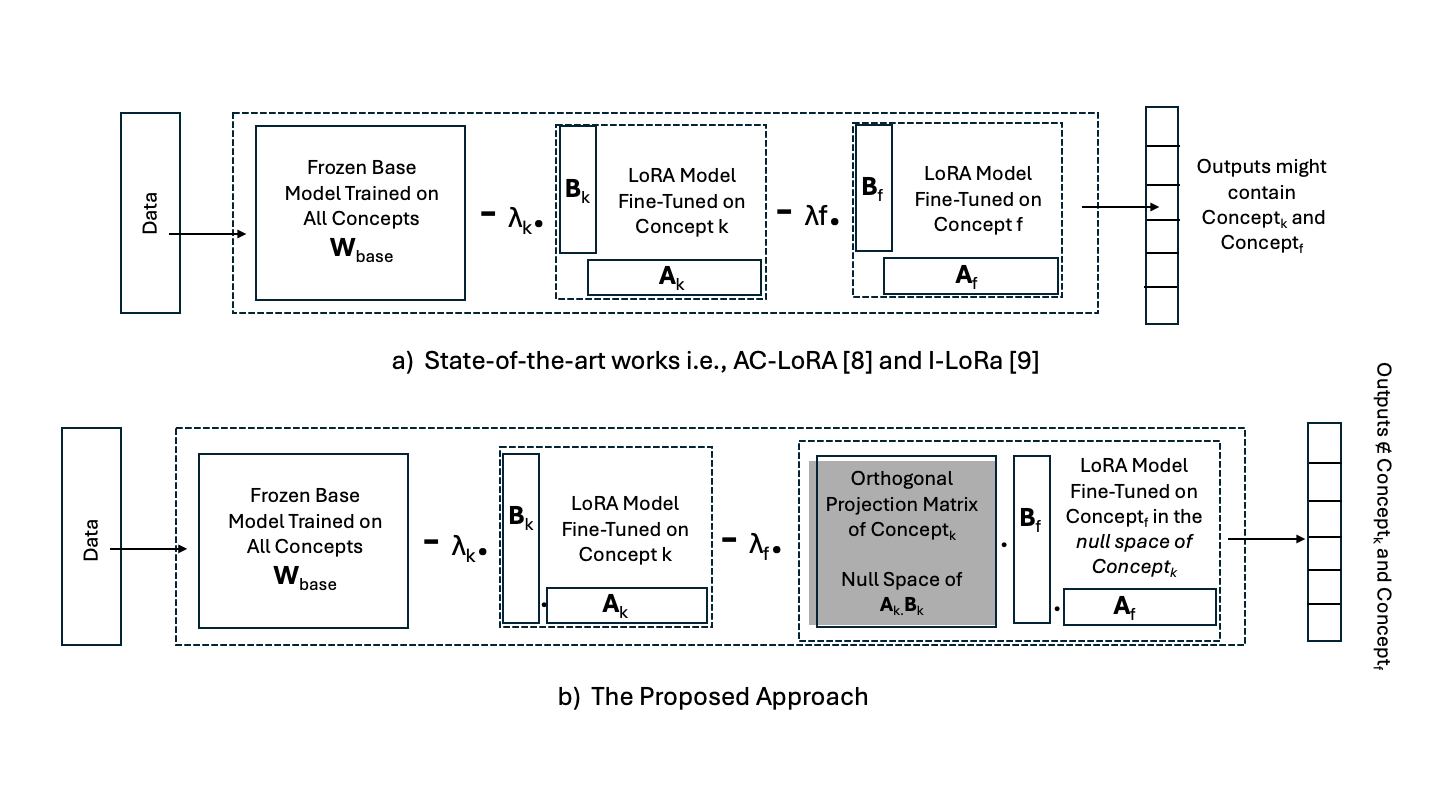}}
  \caption{High-level architecture comparision of the related works (top) and the proposed approach (bottom).}
    \label{fig:architecture}
\end{figure*}

\section{Background and Related Work}
As we observed in our recent work \cite{nasir2026}, the classical machine unlearning methods based on retraining \cite{bourtoule2021machine}, Fisher Information Matrices \cite{golatkar2020eternal}, or influence functions \cite{koh2017understanding} are expensive under sequential deletion requests. LoRA \cite{hu2021lora} improves efficiency by learning low-rank updates ($\Delta W_k = A_k B_k$, where $A_k \in \mathbb{R}^{d_{in} \times r}$ and $B_k \in \mathbb{R}^{r \times d_{out}}$ are the trainable low-rank factor matrices for task $k$) for $N$ target concepts. Ideally, these adapters are statically fused:
\begin{equation}
W_{fused} = W_{base} - \lambda \sum_{k=1}^N \Delta W_k,
\label{eq:fuse}
\end{equation}
However, unconstrained optimization causes the weight distributions of sequential adapters to heavily overlap. Statically fusing them in (\ref{eq:fuse}) results in severe parameter collision, degrading the model's core intelligence on retained data.

To avoid parameter collision, recent methods use dynamic MoE routing \cite{dou2023loramoe}. Frameworks such as AC-LoRA \cite{aclora2025} and I-LoRA \cite{ilora2024} isolate task-specific updates and use a routing network, parameterized by $\theta_G$, to evaluate the input $x$:
$$ \alpha(x) = \text{Softmax}(G(x; \theta_G)) $$
where $\alpha_k(x) \in [0, 1]$ represents the activation weight for the $k$-th adapter. The modified forward pass becomes a dynamically weighted sum:
$$ h = x W_{base} + \sum_{k=1}^N \alpha_k(x) \cdot x (A_k B_k) $$
By pushing $\alpha_k(x) \to 0$ for non-target data, these methods protect the base model from interference. However, they require $\mathcal{O}(N)$ memory to keep all adapters loaded and add routing latency at inference time.

\textbf{Our contribution.} Our framework removes this deployment bottleneck by mapping sequential unlearning requests to orthogonal subspaces during training, ensuring that updates ($\Delta W_k$) do not collide. Because the updates are orthogonal, they can be fused into the base weights using (\ref{eq:fuse}), eliminating the router and dynamic summation at inference time.

\section{Proposed Method}
Consider a pre-trained deep neural network such as a ResNet. While LoRA can be applied across multiple layers, we formulate the framework on a single target layer (e.g., the final classification head) for clarity.

As shown in Fig. \ref{fig:architecture}, let $W_{base} \in \mathbb{R}^{d_{in} \times d_{out}}$ denote the frozen weights of the target layer. We receive a sequence of $N$ unlearning tasks, where task $k$ aims to forget data distribution $D_k$. For each task, we introduce a LoRA parameterized by $A_k \in \mathbb{R}^{d_{in} \times r}$ and $B_k \in \mathbb{R}^{r \times d_{out}}$, where $r \ll \min(d_{in}, d_{out})$, and obtain $\Delta W_k = A_k B_k$. The goal is to fuse these updates as in (\ref{eq:fuse}) without degrading accuracy on the retained distribution $D_{retain}$.

\subsection{Subspace Projection}
To prevent parameter collision between sequential updates, we must ensure that the update for task $k$ does not interfere with the principal geometric directions utilized by tasks $1$ through $k-1$. In order to achieve this, for a given update $\Delta W_k$, we perform SVD and obtain the left and right singular matrices $U_k$  and $V_k$ as follows:
\begin{equation}
\Delta W_k = U_k S_k V_k^T,
\label{eq:svd of W_k}
\end{equation}
We then extract the top $r$ left singular vectors, $U_{k, r} \in \mathbb{R}^{d_{in} \times r}$, which represent the principal input directions spanning the unlearning trigger for task $k$. We maintain a cumulative orthogonal projection matrix $P \in \mathbb{R}^{d_{in} \times d_{in}}$, initialized as the identity matrix $I$. After each task $k$ is optimized, the projection matrix is recursively updated by removing the newly occupied subspace:
\begin{equation}
P_k = P_{k-1} - (U_{k, r} U_{k, r}^T),
\label{eq:projection_update}
\end{equation}
where $P_0 = I$. 

\subsection{Subspace Projection Algorithms}
Once the cumulative projection matrix $P_{k-1}$ is established, it defines the safe mathematical null space for the $k$-th unlearning task. Let $\mathcal{L}_{CE}$ denote the standard Cross-Entropy loss. For a deletion request targeting the data distribution $D_k$, the objective is to train the adapter to perfectly capture the conceptual triggers of $D_k$, such that its isolated parametric contribution can be subsequently subtracted. We propose two distinct algorithmic approaches for applying the geometric constraint $P_{k-1}$ to the task-specific adapter:

\textbf{1) Post-Hoc Orthogonal Projection:} 
In this approach, the $k$-th LoRA module ($\Delta W_k = A_k B_k$) is initially optimized freely without any structural constraints. The adapter is trained to minimize the empirical risk on the forget distribution $D_k$:
$$ \min_{A_k, B_k} \mathbb{E}_{(x,y) \sim D_k} [\mathcal{L}_{CE}(f(x; W_{base} + A_k B_k), y)] $$
The optimizer freely navigates the loss landscape, converging to unconstrained local minima $\hat{A}_k$ and $\hat{B}_k$. After training concludes, we apply the geometric constraint post-hoc by projecting the learned input weight matrix onto the safe subspace:
$$ \tilde{A}_k = P_{k-1} \hat{A}_k $$
The final, safely fused update becomes $\Delta W_k = \tilde{A}_k \hat{B}_k$. While computationally straightforward, deep neural networks feature highly non-convex loss landscapes. Projecting a converged, unconstrained weight matrix onto a lower-dimensional subspace post-training often arbitrarily displaces the weights to a sub-optimal point, which can severely degrade the target unlearning efficacy.

\textbf{2) In-Training Constrained Optimization:} 
To resolve the sub-optimality of post-hoc corrections, we formulate the unlearning update as a constrained optimization problem. In this second approach, we embed the cumulative projection matrix directly into the forward pass of the $k$-th LoRA module from the onset of training. The effective input matrix is continuously defined as $\tilde{A}_k = P_{k-1} A_k$. 

The unlearning objective for the $k$-th concept is thus strictly bounded by the projection matrix during the forward pass:
$$ \min_{A_k, B_k} \mathbb{E}_{(x,y) \sim D_k} [\mathcal{L}_{CE}(f(x; W_{base} + P_{k-1} A_k B_k), y)] $$ This approach intrinsically acts as a learning in the null space of the previous tasks. We formally prove this geometric constraint via the following lemma:

\begin{lemma}\textbf{Gradient Feasibility}
Let $\mathcal{L}(\tilde{A}_k, B_k)$ be the training loss for task $k$, where $\tilde{A}_k = P_{k-1}A_k$. Then:
$$ \nabla_{A_k}\mathcal{L} = P_{k-1}\nabla_{\tilde{A}_k}\mathcal{L} $$
Therefore, every gradient step on $A_k$ strictly lies in $\mathrm{Range}(P_{k-1})$. In particular, $(I - P_{k-1})\nabla_{A_k}\mathcal{L} = 0$.
\end{lemma}

\begin{proof}
Applying the differential chain rule to $\tilde{A}_k = P_{k-1}A_k$ yields:
$$ d\tilde{A}_k = P_{k-1}dA_k $$
The differential of the loss $\mathcal{L}$ with respect to the constrained parameters is defined by the Frobenius inner product:
$$ d\mathcal{L} = \langle \nabla_{\tilde{A}_k}\mathcal{L}, d\tilde{A}_k \rangle $$
Substituting $d\tilde{A}_k$ into the inner product:
$$ d\mathcal{L} = \langle \nabla_{\tilde{A}_k}\mathcal{L}, P_{k-1}dA_k \rangle = \langle P_{k-1}^T \nabla_{\tilde{A}_k}\mathcal{L}, dA_k \rangle $$
Because the orthogonal projection matrix is symmetric ($P_{k-1}^T = P_{k-1}$), this simplifies to:
$$ d\mathcal{L} = \langle P_{k-1} \nabla_{\tilde{A}_k}\mathcal{L}, dA_k \rangle $$
By the definition of the differential $d\mathcal{L} = \langle \nabla_{A_k}\mathcal{L}, dA_k \rangle$, it directly follows that $\nabla_{A_k}\mathcal{L} = P_{k-1} \nabla_{\tilde{A}_k}\mathcal{L}$.
\end{proof}

Because the optimizer is mathematically restricted from exploring or updating dimensions utilized by prior tasks, the adapter converges to a true local minimum strictly within the feasible orthogonal subspace. This maximizes unlearning efficacy while guaranteeing zero interference with historical tasks.

\begin{table*}[htbp]
\centering
\caption{Continual Unlearning Performance on ResNet-20 (CIFAR-100). Colors indicate: \textcolor{blue}{\textbf{Blue}} (Excellent Retain Preservation), \textcolor{red}{\textbf{Red}} (Catastrophic Retain Collapse), \textcolor{teal}{\textbf{Green}} (Successful Unlearning), and \textcolor{orange}{\textbf{Orange}} (Limited Unlearning Efficacy).}
\label{tab:resnet_results}
\resizebox{\textwidth}{!}{%
\begin{tabular}{llcccc}
\toprule
\multirow{2}{*}{\textbf{Tasks ($N$)}} & \multirow{2}{*}{\textbf{Method}} & \multicolumn{2}{c}{\textbf{$\lambda=0.5$}} & \multicolumn{2}{c}{\textbf{$\lambda=1.0$}} \\
\cmidrule(lr){3-4} \cmidrule(lr){5-6}
& & \textbf{Retain Acc. ($\uparrow$)} & \textbf{Forget Acc. ($\downarrow$)} & \textbf{Retain Acc. ($\uparrow$)} & \textbf{Forget Acc. ($\downarrow$)} \\
\midrule

\multirow{3}{*}{$N=5$} 
& Static / AC-LoRA \cite{aclora2025}, I-LoRA \cite{ilora2024} & 56.23\% & 5.60\% & 48.78\% & \textcolor{teal}{0.00\%} \\
& MoE Router \cite{dou2023loramoe} & \textcolor{blue}{59.80\%} & \textcolor{orange}{51.20\%} & \textcolor{blue}{59.80\%} & \textcolor{orange}{51.20\%} \\
& \textbf{In-Training Approach (Ours)} & \textcolor{blue}{59.04\%} & \textcolor{orange}{37.60\%} & \textcolor{blue}{58.02\%} & \textcolor{orange}{33.80\%} \\
\midrule

\multirow{3}{*}{$N=10$} 
& Static / AC-LoRA \cite{aclora2025}, I-LoRA \cite{ilora2024} & 51.89\% & 13.80\% & \textcolor{red}{37.58\%} & \textcolor{teal}{1.20\%} \\
& MoE Router \cite{dou2023loramoe} & \textcolor{blue}{59.62\%} & \textcolor{orange}{57.00\%} & \textcolor{blue}{59.62\%} & \textcolor{orange}{57.00\%} \\
& \textbf{In-Training Approach (Ours)} & \textcolor{blue}{58.68\%} & \textcolor{orange}{50.70\%} & \textcolor{blue}{57.19\%} & \textcolor{orange}{48.40\%} \\
\midrule

\multirow{3}{*}{$N=20$} 
& Static / AC-LoRA \cite{aclora2025}, I-LoRA \cite{ilora2024} & 45.30\% & 10.35\% & \textcolor{red}{25.19\%} & \textcolor{teal}{0.25\%} \\
& MoE Router \cite{dou2023loramoe} & \textcolor{blue}{60.46\%} & \textcolor{orange}{55.00\%} & \textcolor{blue}{60.46\%} & \textcolor{orange}{55.00\%} \\
& \textbf{In-Training Approach (Ours)} & \textcolor{blue}{59.39\%} & \textcolor{orange}{50.75\%} & \textcolor{blue}{58.23\%} & \textcolor{orange}{48.00\%} \\
\midrule

\multirow{3}{*}{$N=30$} 
& Static / AC-LoRA \cite{aclora2025}, I-LoRA \cite{ilora2024} & \textcolor{red}{35.21\%} & 10.60\% & \textcolor{red}{\textbf{12.74\%}} & \textcolor{teal}{\textbf{0.77\%}} \\
& MoE Router \cite{dou2023loramoe} & \textcolor{blue}{59.49\%} & \textcolor{orange}{59.10\%} & \textcolor{blue}{59.49\%} & \textcolor{orange}{59.10\%} \\
& \textbf{In-Training Approach (Ours)} & \textcolor{blue}{\textbf{58.76\%}} & \textcolor{orange}{56.13\%} & \textcolor{blue}{\textbf{58.06\%}} & \textcolor{orange}{53.83\%} \\
\bottomrule
\end{tabular}%
}
\end{table*}

\section{Experiments and Results}
To validate the proposed algorithms, we conduct empirical evaluations across two distinct domains: a fundamental 2-layer Multi-Layer Perceptron (MLP) on the MNIST dataset, and a highly non-convex deep architecture (ResNet-20) on the CIFAR-100 dataset. For the continual unlearning protocol, the models are sequentially tasked with forgetting $N$ distinct classes (up to $N=3$ for MNIST and $N=30$ for CIFAR-100).

\begin{figure*}[htbp]
    \centering
    \includegraphics[width=0.9\textwidth]{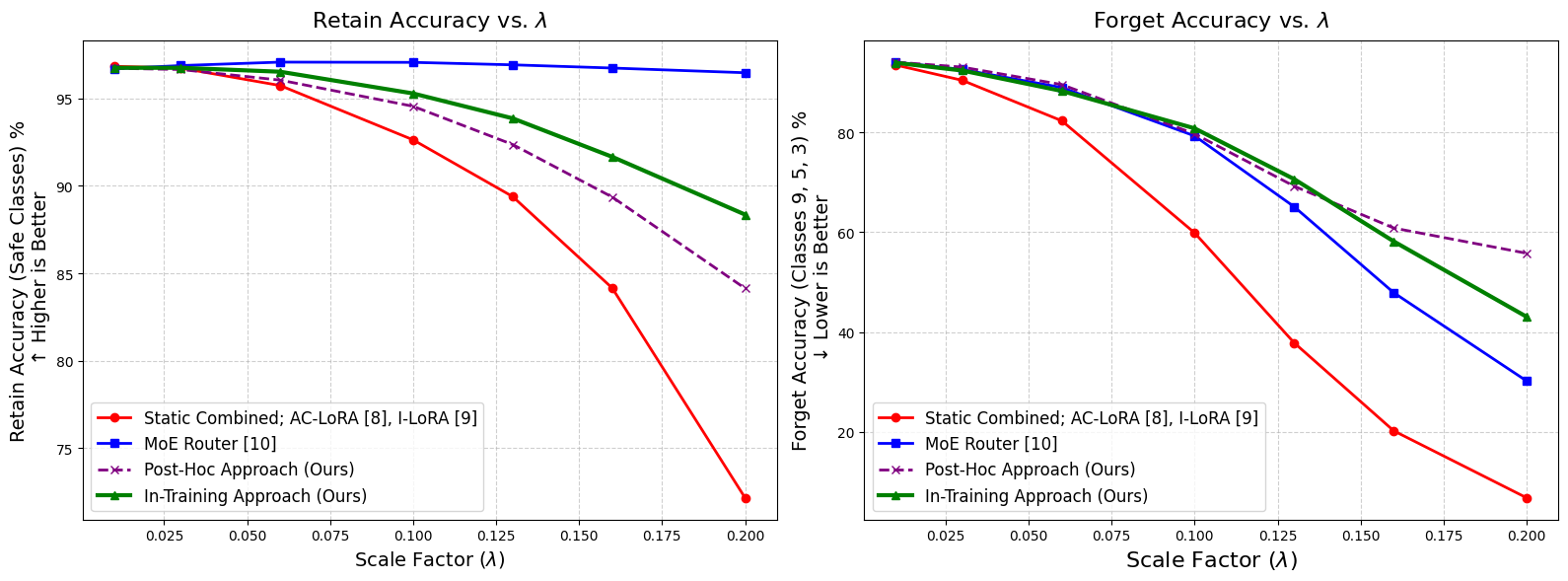}
    \caption{Operational trade-off trajectories for continual unlearning by sequentially forgetting 3 classes 9, 5 and 3 ($N=3$) on the MNIST dataset.}
    \label{fig:mnist_results}
\end{figure*}

\subsection{Results on MNIST Dataset}
We train a simple MLP on MNIST and sequentially unlearn classes 9, 5, and 3. As shown in Fig. \ref{fig:mnist_results}, increasing $\lambda$ exposes the weakness of unconstrained LoRA fusion: retained accuracy collapses as overlapping updates interfere. In contrast, the MoE router \cite{dou2023loramoe} and both proposed orthogonalization methods preserve a stable baseline while reducing forget accuracy to the 10--50\% range.


\subsection{Results on CIFAR-100 Dataset}
For the high-dimensional evaluation, we use a ResNet-20 pre-trained on CIFAR-100. We freeze the convolutional backbone and inject LoRA only into the final classification layer ($W_{base} \in \mathbb{R}^{64 \times 100}$). The model then undergoes $N=30$ sequential unlearning tasks, each targeting one class. Each task uses a low-rank update with $A_k \in \mathbb{R}^{64 \times 4}$ and $B_k \in \mathbb{R}^{100 \times 4}$. In the proposed in-training approach, the adapters are trained sequentially and the projection matrix $P_{k-1}$ is updated after each task via SVD.

Table \ref{tab:resnet_results} shows the core trade-off. Under extreme pressure ($\lambda = 1.0$), the Static Combined baselines \cite{ilora2024,aclora2025} achieves near-zero forget accuracy ($0.77\%$ at $N=30$) but collapses retain accuracy to $12.74\%$. The MoE Router \cite{dou2023loramoe} preserves retain accuracy near $59.5\%$ across capacities, but keeps target-class accuracy high ($\sim 59.10\%$).

\begin{figure}[htbp]
    \centering
    \includegraphics[width=0.9\linewidth]{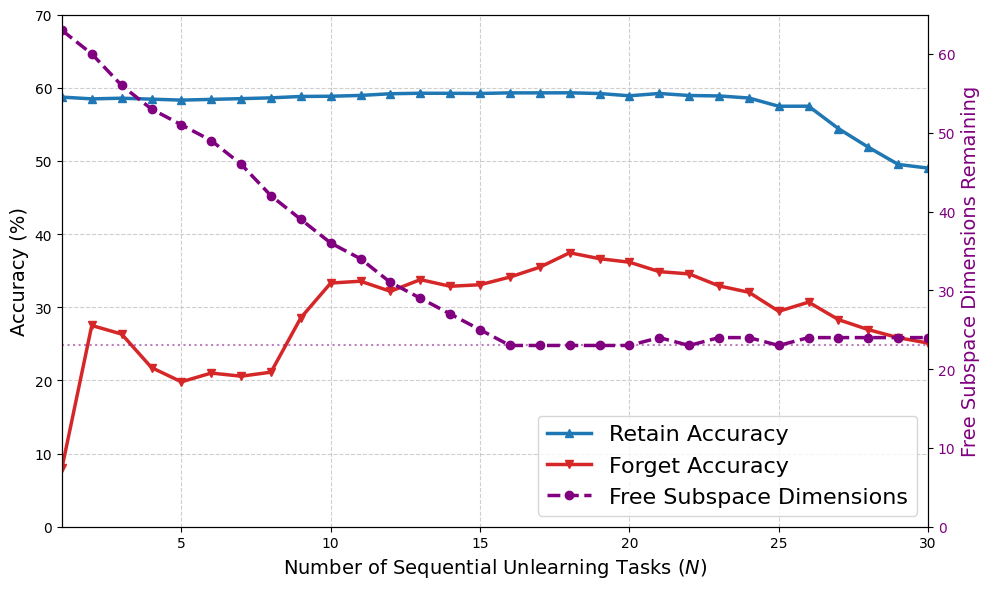}
    \caption{Correlation between available orthogonal subspace dimensions and continual unlearning performance on CIFAR-100 (ResNet-20) for $\lambda = 2.0$.}
    \label{fig:subspace_accuracy}
\end{figure}

\subsection{Performance of the Proposed In-Training Formulation}

Our in-training approach avoids the catastrophic forgetting seen in static baselines while bypassing routing overhead. At $N=20$ and $N=30$, it keeps retain accuracy stable (e.g., $58.06\%$ at $N=30, \lambda=1.0$). The trade-off is reduced unlearning depth under extreme capacity, with forget accuracy rising to $\sim 53.83\%$ at $N=30$. The standard static LoRA baseline suffers severe parameter collision: low forget accuracy is obtained only by collapsing retain accuracy to $12.74\%$. Dynamic routing prevents this collapse, but requires loading $30$ LoRA modules and evaluating a router on every forward pass. Our in-training approach captures the same retain stability as MoE routing ($\sim 58.0\%$) while keeping the updates geometrically isolated.

\subsection{Subspace Saturation and Performance Correlation}
Figure \ref{fig:subspace_accuracy} shows the relation between available subspace and performance. For roughly the first 15 tasks, the free orthogonal subspace depletes linearly and retain accuracy stays near $58.5\%$. Around $N=16$, the remaining space saturates near 23 dimensions, limiting further unlearning and causing forget accuracy to plateau. Retain accuracy stays protected until extreme overload ($N > 24$).
\section{Conclusion and Future Work}
In this letter, we introduced an in-training constrained formulation for continual unlearning that embeds SVD-based orthogonal projections directly into the adapter forward pass. Results on MNIST and CIFAR-100 show that the method prevents parameter collision across up to $N=30$ sequential tasks while preserving strong retained accuracy.

Future work will investigate non-linear projection techniques that can better capture complex task geometry, as well as min-max optimization formulations that more explicitly balance effective unlearning against retained model performance.

\balance

\end{document}